\documentclass[10pt, conference, compsocconf]{IEEEtran}
\renewcommand{\baselinestretch}{0.90} 

\let\oldbibliography\thebibliography
\renewcommand{\thebibliography}[1]{\oldbibliography{#1}
\setlength{\itemsep}{0pt}}

\usepackage[tight,footnotesize]{subfigure}
\usepackage{graphicx}
\usepackage{url}
\usepackage{tabu}
\usepackage{makecell}
\usepackage{tabularx}
\usepackage{lipsum}
\usepackage{array,multirow,graphicx}
\usepackage{multicol}
\usepackage[linesnumbered,lined,commentsnumbered, ruled,vlined]{algorithm2e}
\usepackage{bm}
\usepackage{booktabs}
\usepackage[font=small,labelfont=bf]{caption}
\usepackage{amsmath}
\usepackage{amssymb}
\usepackage{amsthm}
\usepackage{color}

\usepackage{tikz}
\usepackage{pgfplots}
\pgfplotsset{width=3.25cm, height=2.4cm, compat=1.5}
\pgfplotsset{every axis/.append style={
                    label style={font=\tiny},
                    tick label style={font=\tiny}  
                    }
}

\DeclareMathOperator{\Rd}{\mathbb{R}}
\DeclareMathOperator{\E}{\mathbb{E}}
\DeclareMathOperator{\Pp}{\mathbb{P}}

\DeclareMathOperator{\F}{\mathcal{F}}

\DeclareMathOperator{\X}{\mathcal{X}}

\DeclareMathOperator{\D}{\mathcal{D}}
\DeclareMathOperator{\SL}{\mathcal{S}}

\DeclareMathOperator{\Ll}{\mathcal{L}}

\DeclareMathOperator{\Dd}{\mathbf{D}}

\DeclareMathOperator{\R}{\mathbf{R}}

\DeclareMathOperator{\GF}{\mathbf{G}_{F}}
\DeclareMathOperator{\DF}{\mathbf{D}_{F}}
\DeclareMathOperator{\GA}{\mathbf{G}_{A}}
\DeclareMathOperator{\DA}{\mathbf{D}_{A}}
\DeclareMathOperator{\GB}{\mathbf{G}_{B}}
\DeclareMathOperator{\DB}{\mathbf{D}_{B}}
\DeclareMathOperator{\XA}{\mathbf{X}_{A}}
\DeclareMathOperator{\NA}{\mathbf{N}_{A}}
\DeclareMathOperator{\XB}{\mathbf{X}_{B}}

\newtheorem{theorem}{Theorem}[section]

\ifCLASSINFOpdf
\else
\fi
\begin{document}
%
\title{Learning to Fuse Music Genres with Generative Adversarial Dual Learning}




%
\author{\IEEEauthorblockN{
Zhiqian Chen\IEEEauthorrefmark{1},
Chih-Wei Wu\IEEEauthorrefmark{2}, 
Yen-Cheng Lu\IEEEauthorrefmark{1},
Alexander Lerch\IEEEauthorrefmark{2} and
Chang-Tien Lu\IEEEauthorrefmark{1}}

\IEEEauthorblockA{\IEEEauthorrefmark{1}Computer Science Department, Virginia Tech\\Email:\{czq,kevinlu,ctlu\}@vt.edu}\IEEEauthorblockA{\IEEEauthorrefmark{2}Center for Music Technology, Georgia Institute of Technology\\Email:\{cwu307, alexander.lerch\}@gatech.edu}
}


\maketitle

\begin{abstract}
FusionGAN is a novel genre fusion framework for music generation that integrates the strengths of generative adversarial networks and dual learning. In particular, the proposed method offers a dual learning extension that can effectively integrate the styles of the given domains. To efficiently quantify the difference among diverse domains and avoid the vanishing gradient issue, FusionGAN provides a Wasserstein based metric to approximate the distance between the target domain and the existing domains. Adopting the Wasserstein distance, a new domain is created by combining the patterns of the existing domains using adversarial learning. Experimental results on public music datasets demonstrated that our approach could effectively merge two genres.
\end{abstract}

%
%

%
\IEEEpeerreviewmaketitle

\section{Introduction}
Computational creativity is a lively research area that focuses on understanding and facilitating human creativity through the implementation of software programs \cite{colton2009computational}. With the rapid advances in data-driven algorithms such as deep learning \cite{hinton2006fast}, the exploration into computational creativity via machine learning approaches has become increasingly popular.  Examples showing the potential of deep learning for creative approaches are the artistic style transfer on images \cite{Gatys2015} and videos \cite{ruder2016artistic}. Music, with its complex hierarchical and sequential structure and its inherent emotional and aesthetic subjectivity, is an intriguing research subject at the core of human creativity. While there has been some work on generative models for music, research that investigates the capabilities of deep learning for creative applications such as style transfer is limited. This paper aims to fill this space by exploring the idea of style fusion in music with generative adversarial dual learning.


In the field of unsupervised generative learning, generative adversarial networks (GAN) \cite{goodfellow2014generative} have recently gained considerable attention.
It is important to note, however, that GANs are designed to learn from a single domain and cannot discover cross-domain knowledge. 
Inspired by dual learning in machine translation, several recent publications propose methods to pair unlabeled data points in different domains by optimizing bi-directional reconstruction errors \cite{yi2017dualgan,kim2017learning,zhu2017unpaired}. Although these methods are designed to address the challenge of multimodal learning, they do not address the problem of domain fusion.  


The goal of this paper is to provide a solution for fusing two or more groups of sequential patterns using unsupervised methods. Specifically, we focus on the problem of generating music with a fused music genre.
To combine genres using generative adversarial learning, we propose a framework for integrating multiple GANs. Unlike previous work, our study targets the creation of an unknown domain by mixing two given domains; a multi-way GAN-based model is proposed to absorb existing patterns and yield a new mixture. 
With the utilization of the Wasserstein measure \cite{arjovsky2017wasserstein}, we pose a distance constraint on the new domain that automatically balances its relation to the given domains. 
As a result, our model is capable of generating a mixed pattern after convergence. The main contributions of this paper are:
\begin{itemize}
    \item \textbf{A novel framework for unsupervised music fusion}: The dual learning scheme is extended to involve multiple domains by leveraging mutual regularization. In this way, the proposed framework enables the new domain to keep equal similarities with the given domains and produce a mixture of existing sequential patterns.
    \item \textbf{A sequence fusion method using GANs}: To apply unsupervised fusion learning, we extend the generative adversarial model so that multiple generators and discriminators can be updated by one another. 
    \item \textbf{Formulation of an objective function indicating fusion progress}: An objective function based on Wasserstein distance is designed to integrate the information from multiple domains and represent learning progress.
\end{itemize}


\section{Related Work}\label{rw}
\textbf{Music Genre Fusion:} Fusion is mostly known as the sub-genre of jazz that emerged in the late '60s, and that combines several musical styles such as funk, rock, and blues with the jazz harmony and improvisation \cite{GeraldDavid2005}. The term can be generalized, however, to any combinations of music genres. The creative combination of two or more genres obviously requires not only intimate knowledge of the stylistic characteristics of the involved genres but also extensive compositional experience and skill to combine genres in a meaningful and satisfying way. 
Recently, Engel et al.~\cite{engel2017neural} proposed a new approach of generating new musical sounds through interpolating the latent space learned by WaveNet \cite{oord2016wavenet} autoencoders. 
Additionally, Gatys et al. demonstrated the Deep Neural Networks' (DNNs) capabilities of learning the artistic styles by blending the style of the training materials to the testing images using ConvNet \cite{Gatys2015}. Both examples show the great potential of using DNNs for merging two abstract concepts without the need for domain knowledge and explicitly defined rules. 

\textbf{Cross-domain GAN:} GAN \cite{goodfellow2014generative}, which has quickly risen to one of the most popular generative approaches, learns patterns without requiring adversarial examples or labels. In CGAN \cite{mirza2014conditional}, the generator $G$ has a conditional parameter $c$ and will learn the conditional distribution of the data. However, those conditions need to be provided manually, somewhat similar to supervised learning. CoGAN \cite{liu2016coupled} is proposed to learn a joint distribution with only samples drawn from the marginal distributions. DiscoGAN \cite{kim2017learning} solves the cross-domain pairing by minimizing bi-directional reconstruction loss, while DualGAN \cite{yi2017dualgan} takes advantage of the dual learning paradigm in machine translation. CycleGAN \cite{zhu2017unpaired} presents a method whose mapping functions are cycle-consistent and proposes a cycle consistency loss function to further reduce the space of possible mapping functions.

Different from the previous studies, our work raises a new problem, i.e., merging two sequential patterns into one. We propose a GAN-based framework that merges two domains without manually tuning the distance between the given domains and the target domain.

\section{Learning to Blend Music with FusionGAN}\label{method}

\subsection{Problem Definition}
This paper aims to characterize music fusion across different genres in an unsupervised manner. Based on the data $\XA$ and $\XB$ from existing music domains $\D_{A}, \D_{B}$ respectively, our goal is to create a new domain $\D_{F}$ which resembles both $\D_{A}$ and $\D_{B}$. In the first phase, the new domain $\D_{F}$ is expected to learn from $\D_{A}, \D_{B}$ and keep an equal distance from both.  Specifically, the generator $\GF$ of $\D_{F}$ obtains feedback from the discriminators $\DA,\DB$ of existing domains $\D_{A}, \D_{B}$, while the discriminator $\DF$ collects data from $\GF$ and  $\D_{A}, \D_{B}$ for iterative updating. The $\GF$ from the new domain minimizes the distance from the domains $\D_{A}, \D_{B}$, respectively, and keeps the same distance from $\D_{A}$ and $\D_{B}$. 

\subsection{FusionGAN}
First, the framework initializes individual GAN models for existing domains. Maximum likelihood estimation using Long Short Term Memory (LSTM) is applied to initialize the sequence generators $\GA(\GB)$. Then $\DA$($\DB$) are trained given the domain data and sequence sampled from $\GA(\GB)$. Then GAN is employed to iteratively enhance the $\GA(\GB)$ and $\DA$($\DB$). 
After pre-training, we build a GAN model for the new domain using the feedback from the models constructed in the pre-training procedure. First, a new domain $\D_{F}$ is initialized randomly and trained by both $\D_{A}$ and $\D_{B}$. Following the dual learning strategy, $\D_{A}$ is enhanced by $\D_{B}$ and $\D_{F}$ in the second phase. Similarly, $\D_{B}$ is improved by $\D_{A}$ and $\D_{F}$. Such learning proceeds until convergence. The framework overview is shown in Figure \ref{fig:fusion_framework}.

\begin{figure}[!hpbt]
    \centering
    \includegraphics[width=2.4in]{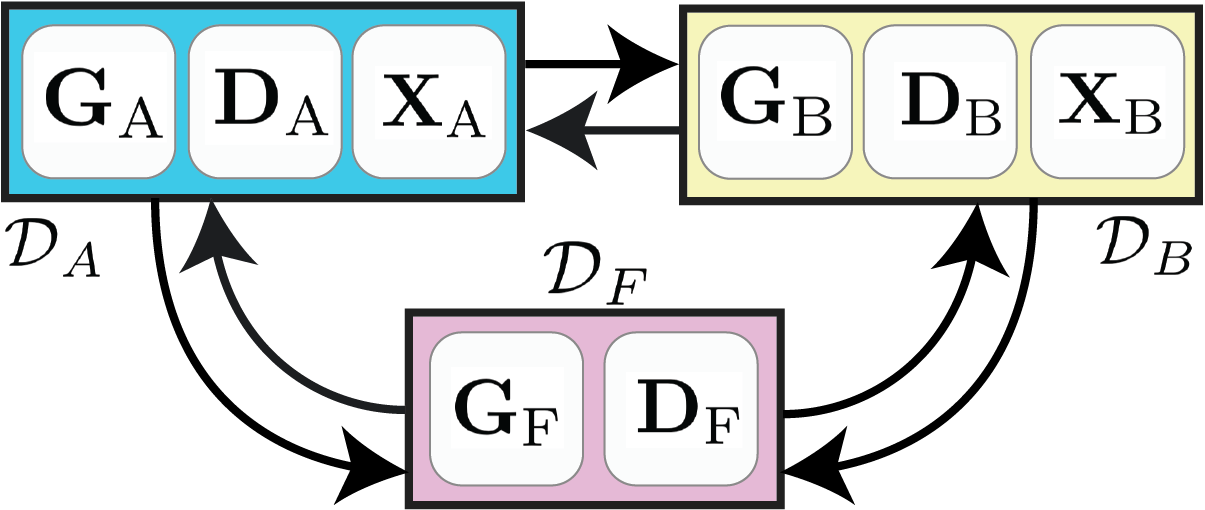}
    \caption{FusionGAN framework}
    \label{fig:fusion_framework}
\end{figure}

In FusionGAN, we employ the Wasserstein-1 distance for all discriminators because it is sensible when the learning distributions are supported by low dimensional manifolds \cite{arjovsky2017wasserstein}. 
To avoid the intractable infimum in Wasserstein-1 distance, we resort to its Kantorovich-Rubinstein duality \cite{villani2008optimal}: $ W\left( \Pp_{r},\Pp_{\theta }\right) =\sup_{\left\| f\right\| _{L}\leq 1}\E_{x\sim \Pp_{r}}\left[ f\left( x\right) \right]-\E_{x\sim \Pp_{\theta }}\left[ f\left( x\right) \right] $.
For the vanilla GAN, the goal is to find the optimal configuration of the parameters $\phi$ of discriminator ($f=\Dd$). When the discriminator is optimized, the maximized Wasserstein distance can be used as reward in the policy gradient process of the generator. Similarly, we define a three-way Wasserstein distance as the objective function: $W\left( \Pp_{\XA},\Pp_{\XB},\Pp_{\theta }\right)=\Ll $ to measure the integrated distance between $\Pp_{\theta}$ and $\Pp_{\XA},\Pp_{\XB}$. This measure should consider the relationship between every pair of domains. The framework consists of three pairs of generators and discriminators, i.e., $\GA-\DA, \GB-\DB, \GF-\DF$, and two input data, i.e., $\XA, \XB$. For each discriminator, there are five possible inputs, i.e., $\XA,\XB, \GA, \GB, \GF$. Collecting all possible inputs w.r.t. the three discriminators, the objective function to maximize is defined as:
{
\footnotesize
\begin{equation}\label{loss}
\begin{aligned}
\Ll &=&&W\left( \Pp_{\XA},\Pp_{\XB},\Pp_{\theta }\right) \\ 
&&&\E_{x\sim \Pp_{\XA}}\left[\DA\left( x \right) \right]-\E_{x\sim \Pp_{\XB}}\left[\DA\left( x \right) \right] - \\
&&& \E_{z\sim p(z)}\left[\DA\left( \GA\left( z\right) \right)\right] - \E_{z\sim p(z)}\left[\DA\left( \GB\left( z\right) \right) \right]-\\
&&&\E_{z\sim p(z)}\left[\DA\left( \GF\left( z\right) \right)\right] - \\
&&&\E_{x\sim \Pp_{\XA}}\left[\DB\left( x \right) \right]+ \E_{x\sim \Pp_{\XB}}\left[\DB\left( x \right) \right]-\\
&&&\E_{z\sim p(z)}\left[\DB\left( \GA\left( z\right) \right)\right] - \E_{z\sim p(z)}\left[\DB\left( \GB\left( z\right) \right) \right]-\\
&&&\E_{z\sim p(z)}\left[\DB\left( \GF\left( z\right) \right)\right] +\\
&&&\E_{x\sim \Pp_{\XA}}\left[\DF\left( x \right) \right]+ \E_{x\sim \Pp_{\XB}}\left[\DF\left( x \right) \right] +\\
&&&\E_{z\sim p(z)}\left[\DF\left( \GA\left( z\right) \right)\right] + \E_{z\sim p(z)}\left[\DF\left( \GB\left( z\right) \right) \right]-\\
&&&\E_{z\sim p(z)}\left[\DF\left( \GF\left( z\right) \right)\right],\\
\end{aligned} 
\end{equation}
}where the first five terms are for $\DA$, the second five terms are for $\DB$, and the last five terms are for $\DF$. This loss is used to update all the generators and discriminators. The following two subsections, \ref{sec:f_fusiongan} and \ref{sec:ab_fusiongan},  will present the update rules using Eq. \ref{loss}.

\subsection{$\D_{F}$ Update of FusionGAN}\label{sec:f_fusiongan}
Removing the terms that are unrelated to $\D_{F}$ in Eq. \ref{loss} (zero derivative w.r.t. $\DF$ and $\GF$), we have the objective function for updating $\DF$ and $\GF$:
{
\footnotesize
\begin{equation}
\begin{aligned}
\Ll_{F} &=&&\E_{x\sim \Pp_{\XA}}\left[\DF\left( x \right) \right]+ \E_{x\sim \Pp_{\XB}}\left[\DF\left( x \right) \right] +\\
&&&\E_{z\sim p(z)}\left[\DF\left( \GA\left( z\right) \right)\right] + \E_{z\sim p(z)}\left[\DF\left( \GB\left( z\right) \right) \right]-\\
&&&\E_{z\sim p(z)}\left[\DF\left( \GF\left( z\right) \right)\right]-\\
&&&\E_{z\sim p(z)}\left[\DA\left( \GF\left( z\right) \right)\right] - \\
&&&\E_{z\sim p(z)}\left[\DB\left( \GF\left( z\right) \right)\right]. \\
\end{aligned} 
\end{equation}
}First, $\Ll_{F}$ is calculated w.r.t. $\GF$ so as to update $\GF$. Considering differentiating $\nabla _{\theta}\Ll_{F}$, the optimal $\GF$ is approximated as its convergence. The proof is provided below to show that this derivative is principled under the optimality assumption.

\begin{theorem}[FusionGAN Optimality Theorem]\label{theorem}
Let $\Pp_{\XA}$ and $\Pp_{\XB}$ be any distribution, Let $\Pp_{\theta}$ be the distribution of $\GF_{\theta}(Z)$ with $Z$, a random variable with density $p$ and $\GF_{\theta}$ a function satisfying $\E_{z\sim p}[L(\theta,z)] < +\infty$, where $L$ is Lipschitz constants. All functions are well-defined. Then there is a solution $\DF:\X\rightarrow\Rd$  to the problem:
{
\footnotesize
\begin{equation}
W\left( \Pp_{\XA},\Pp_{\XB},\Pp_{\theta }\right)=\max _{\left\| \DF\right\| _{L}\leq 1} \Ll_{F}.
\end{equation}} and we have:
{
\footnotesize
\begin{equation}\label{opttheorem}
\begin{aligned}
\nabla _{\theta }\Ll_{\GF} =&\nabla _{\theta }W\left( \Pp_{\XA},\Pp_{\XB},\Pp_{\theta}\right)&\\
=&-\E_{z\sim p(z)} \nabla _{\theta } \sum_{i\in \{A,B,F\}}\Dd_{i}\left( \GF\left( z\right) \right)&\\
\end{aligned} 
\end{equation}
}

\end{theorem}

\begin{proof}[Proof of Theorem \ref{theorem}]
Let us define $ \Ll_{F}= V(\tilde{\DF},\theta)$ where $\tilde{\DF}$ lies in $\F=\{\tilde{\DF}:\X \rightarrow \Rd, \tilde{\DF} \in C_{b}(\X), \left\|\DF\right\| _{L}\leq 1 \}$ and $\theta \in \Rd_{d}$. By the Kantorovich-Rubinstein duality,  there is an $f \in \F$ that satisfies: 
{
\footnotesize
\begin{equation*}
    W\left( \Pp_{\XA},\Pp_{\XB},\Pp_{\theta }\right) =\sup_{\tilde{\DF}\in\F}V(\tilde{\DF},\theta)=V(\DF,\theta).
\end{equation*}
}
Let us define $X^{*}(\theta)=\{\DF\in\F: V(\DF,\theta)=W\left( \Pp_{\XA},\Pp_{\XB},\Pp_{\theta }\right)\}$, which shows that $X^{*}(\theta)$ is non-empty. According to Theorem 1 in \cite{arjovsky2017wasserstein}: (1) if $\GF$ is continuous in $\theta$, so is $W\left( \Pp_{\XA},\Pp_{\XB},\Pp_{\theta }\right)$, (2) If $\GF$  is locally Lipschitz and satisfies $\E_{z\sim p}[L(\theta,z)] < +\infty$, then $W\left( \Pp_{\XA},\Pp_{\XB},\Pp_{\theta }\right)$ is continuous and differentiable for any $\DF \in X^{*}(\theta)$ when both terms are well-defined. Let $\DF \in X^{*}(\theta)$, which we knows exists since $X^{*}(\theta)$is non-empty for all $\theta$. Then:
{\footnotesize
\begin{equation}\label{wdis}
\begin{aligned}
&\nabla _{\theta }W\left( \Pp_{\XA},\Pp_{\XB},\Pp_{\theta }\right)&\\
=&\nabla _{\theta }V(\DF,\theta)=\nabla _{\theta }\Ll_{F}&\\
=&-\nabla _{\theta }\E_{z\sim p(z)}\lbrack \DF\left( \GF\left( z\right) \right)+\DA\left( \GF\left( z\right) \right) +\DB\left( \GF\left( z\right) \right) \rbrack, &
\end{aligned}
\end{equation}}since $\DF$ is 1-Lipschitz. Furthermore, $\GF_{\theta}(z)$ is locally Lipschitz as a function of ($\theta, z$). Therefore, $\GF_{\theta}(z)$ is locally Lipschitz on ($\theta, z$) with constants $L(\theta, z)$. By Radamacher’s Theorem, $\DF(\GF_{\theta}(z))$ has to be differentiable almost everywhere for ($\theta, z$) jointly. Rewriting this, the set $A = \{(\theta, z) : \DF  \circ \GF $ is not differentiable$\}$ has measure 0. By Fubini’s Theorem, this implies that for almost every $\theta$ the section $A_{\theta} = \{z : (\theta, z) \in A\}$ has measure 0. Let's fix a $\theta_{0}$ such that the measure of $A_{\theta_{0}}$ is null. For this $\theta_{0}$ we have $\nabla _{\theta}\DF(\GF_{\theta}(z))|\theta_{0}$ is well-defined for almost any $z$, and since $p(z)$ has a density, it is defined $p(z)$ almost everywhere. Given the condition $\E_{z\sim p}[L(\theta,z)] < +\infty$, we have
{\footnotesize
\begin{equation}
    \E_{z\sim p(z)}[\left\|\nabla _{\theta}\DF(\GF_{\theta}(z))\right\|] \leq \E_{z\sim p(z)}[L(\theta_{0},z)] < +\infty,
\end{equation}
}so $\E_{z\sim p(z)}[\nabla _{\theta}f(g_{\theta}(z))] $ is well-defined for almost every $\theta_{0}$. To keep the notation simple, we leave it implicit in the following equation that the $\E$ subjects to $z\sim p(z)$.  :
{\footnotesize
\begin{equation}\label{wdisapp}
\begin{aligned}
&\frac{1}{\left\| \theta-\theta_{0} \right\|}\sum_{i} \bigg[\E[\Dd_{i}(\GF_{\theta}(z))] -\E[\Dd_{i}(\GF_{\theta_{0}}(z))] -\\
&\langle(\theta-\theta_{0}),\E\lbrack\nabla _{\theta}[\Dd_{i}(\GF_{\theta}(z))]\rbrack|_{\theta_{0}}\rangle \bigg]\\
=&\frac{1}{\left\| \theta-\theta_{0} \right\|}\sum_{i} \E \bigg[ \Dd_{i}(\GF_{\theta}(z)) -\Dd_{i}(\GF_{\theta_{0}}(z)) -\\
&\langle(\theta-\theta_{0}),\nabla _{\theta}[\Dd_{i}(\GF_{\theta}(z))]|_{\theta_{0}}\rangle \bigg]\\
\leq & \frac{1}{\left\| \theta-\theta_{0} \right\|}\sum_{i} \big\| \left\| \theta-\theta_{0} \right\|L(\theta_{0},z) + \left\| \theta-\theta_{0} \right\|\cdot\|\nabla _{\theta}[\Dd_{i}(\GF_{\theta}(z))]\big\|\\
\leq & \sum_{i}2L(\theta_{0},z) = 6L(\theta_{0},z),
\end{aligned}
\end{equation}}where $i \in \{A,B,F\}$. By differentiability, the term inside the integral converges $p(z)$ to 0 as $\theta \rightarrow \theta_{0}$. Furthermore, and since $\E_{ z \sim p(z)} [6L(\theta_{0},z) ] < +\infty$, we get by dominated convergence that Eq. \ref{wdis} converges to 0 as $\theta \rightarrow \theta_{0}$. So the result of Eq. \ref{wdis} equals to: 
{\footnotesize
\begin{equation*}
-\E_{z\sim p(z)}\lbrack \nabla _{\theta }\DF\left( \GF\left( z\right) \right)+\nabla _{\theta }\DA\left( \GF\left( z\right) \right) +\nabla _{\theta }\DB\left( \GF\left( z\right) \right) \rbrack.
\end{equation*}
}
\end{proof}

Therefore, the update rules for the parameters $\theta$ of $\GF$ is Eq. \ref{opttheorem}.
\begin{figure}[!hpbt]
    \centering
    \includegraphics[width=2.2in]{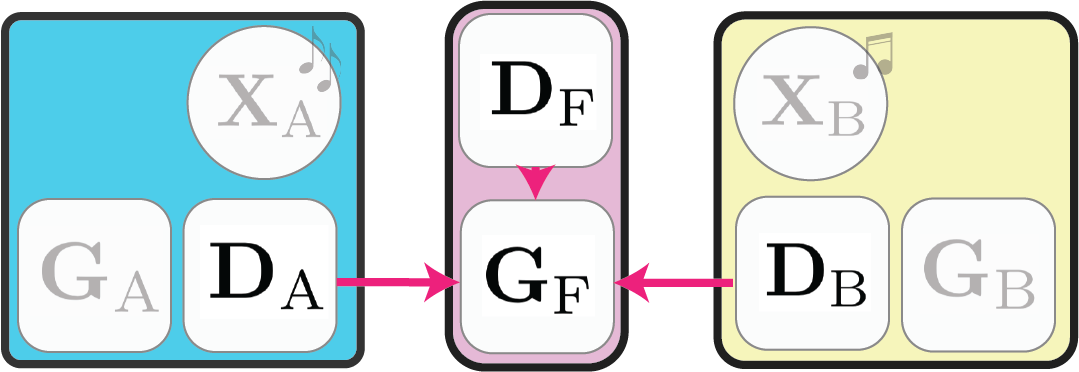}
    \caption{FusionGAN: update $\GF$}\label{fig:fg_update}
\end{figure}

The physical meaning is that all three discriminators can offer feedback to generator $\GF$ as shown in Fig. \ref{fig:fg_update}. The left cyan rectangle indicates $\D_{A}$, the middle pink one is $\D_{F}$, and the right yellow one $\D_{B}$. Similarly, the gradient w.r.t. the parameters $\phi$ of $\DF$ is:
{
\footnotesize
\begin{equation}\label{fd_update}
\begin{aligned}
\nabla_{\phi}\Ll_{F} &=&&\nabla_{\phi}\big[\E_{x\sim \Pp_{\XA}}\left[\DF\left( x \right) \right]+ \E_{x\sim \Pp_{\XB}}\left[\DF\left( x \right) \right] +\\
&&&\E_{z\sim p(z)}\left[\DF\left( \GA\left( z\right) \right)\right] + \E_{z\sim p(z)}\left[\DF\left( \GB\left( z\right) \right) \right]-\\
&&&\E_{z\sim p(z)}\left[\DF\left( \GF\left( z\right) \right)\right]\big].\\
\end{aligned} 
\end{equation}
}

\begin{figure}[!hpbt]
    \centering
    \includegraphics[width=2.2in]{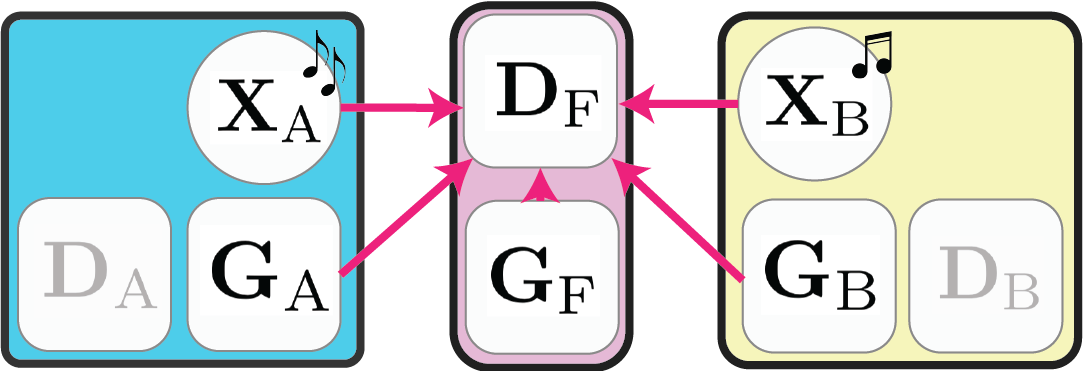}
    \caption{FusionGAN: update $\DF$}
    \label{fig:df_update}
\end{figure}

As shown in Fig. \ref{fig:df_update}, the update for $\DF$ actually implicitly balances itself between $\D_{A}$ and $\D_{B}$. However, this update does not cover the degree of blending which may lead to an unbalanced mixture.  Specifically, a good learning process for the discriminator $\DF$ should consider two factors: (1) minimizing the distance with $\D_{A}$, $\D_{B}$ and $\D_{F}$ simultaneously and (2) maintaining a equal distance from $\D_{A}$ and $\D_{B}$ so as to satisfy both domains. To satisfy (2) and push the discriminator $\DF$ to have a well-proportioned blending, an additional constraint is employed to further balance the ratio of existing domains, i.e., $\D_{A}, \D_{B}$:
{
\footnotesize
\begin{equation}\label{fd_update_add}
\begin{aligned}
\Ll_{F-bal} &=&& \left\| \E_{z\sim p(z)}\left[\DF\left( \GA\left( z\right) \right)\right] - \E_{z\sim p(z)}\left[\DF\left( \GB\left( z\right) \right) \right]\right\| + \\
&&& \left\| \E_{x\sim \Pp_{\XA}}\left[\DF\left( x \right) \right] - \E_{x\sim \Pp_{\XB}}\left[\DF\left( x \right) \right]\right\|.
\end{aligned} 
\end{equation}
}

\subsection{$\D_{A}$ and $\D_{B}$ Update of FusionGAN}\label{sec:ab_fusiongan}
After updating  $\D_{F}$, FusionGAN will improve $\D_{A}$ and $\D_{B}$. Since $\D_{A}$ and $\D_{B}$ are symmetric in the framework, we only discuss $\D_{A}$ as shown in Fig. \ref{fd_update_add}. Keeping related terms in Eq. \ref{loss}, we get the loss function for $\D_{A}$:
{
\footnotesize
\begin{equation}\label{aloss}
\begin{aligned}
\Ll_{A} &=&&\E_{x\sim \Pp_{\XA}}\left[\DA\left( x \right) \right]-\E_{x\sim \Pp_{\XB}}\left[\DA\left( x \right) \right] - \\
&&& \E_{z\sim p(z)}\left[\DA\left( \GA\left( z\right) \right)\right] - \E_{z\sim p(z)}\left[\DA\left( \GB\left( z\right) \right) \right]-\\
&&&\E_{z\sim p(z)}\left[\DA\left( \GF\left( z\right) \right)\right] -\\
&&&\E_{z\sim p(z)}\left[\DB\left( \GA\left( z\right) \right)\right] +\\
&&&\E_{z\sim p(z)}\left[\DF\left( \GA\left( z\right) \right)\right].
\end{aligned} 
\end{equation}
}
\begin{figure}[!hpbt]
    \centering
    \includegraphics[width=2.1in]{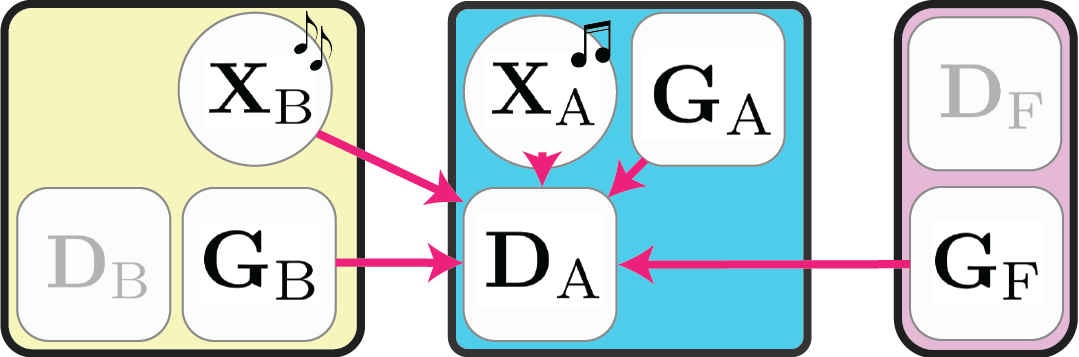}
    \caption{FusionGAN: update $\DA$}\label{fig:ad_update}
\end{figure}

Taking the derivative of $\Ll_{A}$ w.r.t. $\DA$, we have
{
\footnotesize
\begin{equation}\label{ad_update}
\begin{aligned}
\nabla_{\DA}\Ll_{A} &=&&\nabla_{\DA}\big[\E_{x\sim \Pp_{\XA}}\left[\DA\left( x \right) \right]-\E_{x\sim \Pp_{\XB}}\left[\DA\left( x \right) \right] - \\
&&& \E_{z\sim p(z)}\left[\DA\left( \GA\left( z\right) \right)\right] - \E_{z\sim p(z)}\left[\DA\left( \GB\left( z\right) \right) \right]-\\
&&&\E_{z\sim p(z)}\left[\DA\left( \GF\left( z\right) \right)\right] \big]. \\
\end{aligned} 
\end{equation}
}

There are four negative terms and one positive term in Eq. \ref{ad_update} w.r.t. $\DA$. The training implicitly increases the positive terms and decreases the value of the negative terms. However, the inequalities in the two groups are not controlled. To further control the distance among the three domains, we propose an inequality constraint for $\DA$:
{
\footnotesize
\begin{equation}\label{ad_update_add}
\begin{aligned}
&\DA\left( \XA\right) \geq\DA\left( \GF\left( z\right)  \right)  \geq\DA\left( \XB\right) 
\end{aligned} 
\end{equation}
}

Then combining Eq. \ref{ad_update_add} into loss function Eq. \ref{ad_update}, it can be rewritten as:
{
\footnotesize
\begin{equation}\label{ad_update_add_rew}
\begin{aligned}
&\Ll_{A-bal} = \|\DA\left(\XA\right) -\DA\left(\GF\left(z\right) \right)\| +\| \DA\left(\GF\left(z\right)\right)-\DA\left(\XB\right)\|.
\end{aligned} 
\end{equation}
}
\begin{figure}[!hpbt]
    \centering
    \includegraphics[width=2.1in]{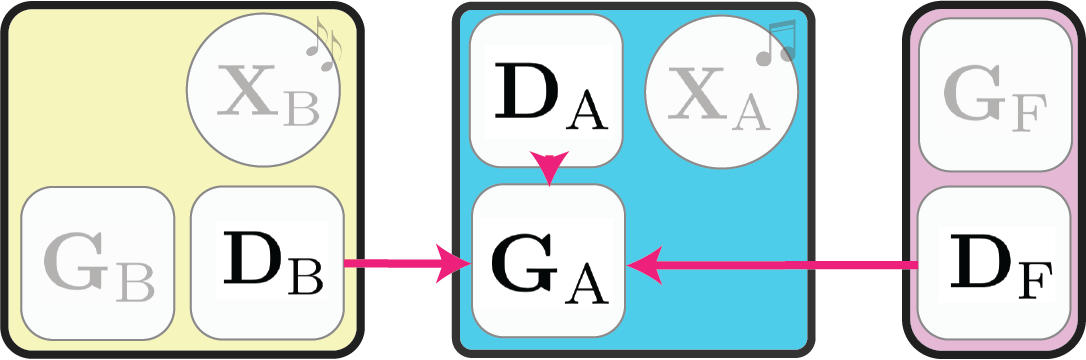}
    \caption{FusionGAN: update $\GA$}\label{fig:ag_update}
\end{figure}

Some inequalities that have been implicitly included in GAN training are not covered in Eq. \ref{ad_update_add}, such as $\DA\left( \XA\right) \geq \DA\left( \GA\right)$. Similar to the $\GF$ update in Fig. \ref{fig:ag_update}, the derivative of $\Ll_{A}$ w.r.t. $\GA$ is:
{
\footnotesize
\begin{equation}\label{ag_update}
\begin{aligned}
\nabla_{\GA}\Ll_{A} &=&& \E_{z\sim p(z)}\nabla_{\GA}\big[-\DA\left( \GA\left( z\right) \right) -\\
&&&\DB\left(\GA\left(z\right)\right) +\DF\left(\GA\left(z\right)\right)\big].
\end{aligned} 
\end{equation}
}

\subsection{Algorithm Description}

{\footnotesize
\begin{algorithm}[!h]
    \caption{FusionGAN}
    \label{algo:fgan}
    \SetAlgoLined
    \KwIn{Input data $\XA,\XB$ from $\D_{A},\D_{B}$}
    \KwOut{a generator and a discriminator: $\GF and \DF$}
    Randomly initialize $\GA, \DA, \GB, \DB, \GF, \DF$, generators/discriminators employ LSTM/TextCNN\;\label{algo:init}
    \tcp{Pre-training}
    Train $\GA$ using maximum likelihood estimation on real data $\XA$\;\label{sgan_start}
    Generate negative samples $\NA$ using $\GA$\; \label{gen_neg}
    Train binary classifier $\DA$ with $\NA$ and $\XA$\; \label{min_entro}
    \Repeat{$\DA$ converges}{ \label{pregan_start}
        \For{$\GA$ training}
        {
            Sample a sequence $\SL_{1:T}=(s_{1},...,s_{T}) \sim \GA$\;
            Derive $Q$ value  as rewards $\R$ from $\DA$\;
            Update parameters $\theta$ of $\GA$ by policy gradient: $\theta \leftarrow \theta+\alpha\nabla_{\theta}\R$  \;
        }
        \For{$\DA$ training}
        {
            Generate negative examples $\sim \GA$\;
            Train $\DA$ with negative samples and $\XA$\;
        }
    }\label{sgan_end}
    Repeat line \ref{sgan_start} to \ref{sgan_end} with $\D_{B}$, i.e., $\DB, \GB, \XB$\;
    Repeat line \ref{sgan_start} to \ref{min_entro} with $\D_{F}$, i.e., $\DF, \GF, \mathbf{X}_{F}=\XA+\XB$\;
    \tcp{FusionGAN training}
    \Repeat{$\DA, \DB, \DF$ converges}{
        \For{$\D_{F}$ training}{\label{fusion_f_start}
            Update parameters of $\DF$ using the sum of Eq. \ref{fd_update} and Eq. \ref{fd_update_add} \label{fusion_df}: $\theta_{\DF} \leftarrow \theta_{\DF}+\alpha\nabla_{\DF}(\Ll_{F}+\Ll_{F-bal}) $ \; 
            Update parameters of $\GF$ using Eq. \ref{opttheorem} $\theta_{\GF} \leftarrow \theta_{\GF}+\alpha\nabla_{\GF}\Ll_{F} $ \;\label{fusion_gf}
        }\label{fusion_f_end}
        \For{$\D_{A}$ training}{\label{fusion_ab_start}
          Update parameters of $\DA$ using the sum of Eq. \ref{ad_update} and Eq. \ref{ad_update_add_rew} :  $\theta_{\DA} \leftarrow \theta_{\DA}+\alpha\nabla_{\DA}(\Ll_{A}+\Ll_{A-bal}) $\; 
        Update parameters of $\GA$ using Eq. \ref{ag_update}: $\theta_{\GA} \leftarrow \theta_{\GA}+\alpha\nabla_{\GA}\Ll_{A} $\;
        }\label{fusion_ab_end}
        Repeat line \ref{fusion_ab_start} to \ref{fusion_ab_end} with $\D_{B}$, i.e., $\DB, \GB, \XB$\;
    }    
\end{algorithm}
}
The Algorithm \ref{algo:fgan}\footnote{https://github.com/aquastar/fusion\_gan} initializes the parameters of the discriminators and generators (line \ref{algo:init}). TextCNN \cite{zhang2015text} and LSTM are employed to build the discriminators and generators respectively. Given  $\XA$, the generator is updated by maximum likelihood estimation (line \ref{sgan_start}). From lines \ref{gen_neg} to \ref{min_entro}, $\DA$ receives the sequences generated by $\GA$ as negative data, and real data as positive samples. Between lines \ref{pregan_start} and \ref{sgan_end}, a GAN procedure is applied to optimize $\DA$ and $\GA$. $\DA$ returns regression scores that are treated as rewards for the policy gradient of $\GA$. Similarly, $\DA$ improves itself by the real world data $\XA$ and the data generated by $\GA$. Our method proceeds to update $\DB$ and $\GB$ by using the same process from lines \ref{sgan_start} to \ref{sgan_end}. $\D_{F}$ repeats the same pre-training procedure from \ref{sgan_start} to \ref{sgan_end} with $\DF, \GF, \mathbf{X}_{F}=\XA+\XB$. From line \ref{fusion_f_start} to \ref{fusion_ab_end}, the framework iteratively updates the generators and discriminators from $\D_{F}, \D_{A}, \D_{B}$. Lines \ref{fusion_f_start} to \ref{fusion_f_end} are for $\D_{F}$, while lines \ref{fusion_ab_start} to \ref{fusion_ab_end} are for $\D_{A}$ and $\D_{B}$. The update rules for the parameters of $\DF$ and $\GF$ have been covered in subsection (\ref{sec:f_fusiongan}), while those for the parameters of $\GA$ and $\DA$ are described in subsection (\ref{sec:ab_fusiongan}). Again, $\D_{B}$ follows the same process as that of $\D_{A}$. The algorithm stops when the discriminators converge.
  
%
\IEEEpeerreviewmaketitle

\section{Evaluation}\label{exp}

\begin{figure*}[!hbpt]
        \centering
        {
            \subfigure{
                \begin{tikzpicture}
                    \begin{axis}[
                        xtick={0,5,10,15},
                        yticklabels={,,},
                        scaled y ticks={base 10:0},
                        legend pos=north east,
                        ymajorgrids=true,
                        grid style=dashed,
                        legend style={font=\tiny},
                    ]
                    \addplot[
                        color=orange,
                        mark=*,
                    ]
                    coordinates {
(0,124689) (1,12475) (2,5047) (3,2573) (4,1493) (5,1024) (6,694) (7,451) (8,355) (9,264) (10,208) (11,167) (12,105) (13,59) (14,50) (15,48) (16,32) (17,25) (18,25) (19,31)                };
                    \legend{jazz}

                    \end{axis}
                \end{tikzpicture}
                \label{}
            }\hspace{0pt}%
            \subfigure{
                \begin{tikzpicture}
                    \begin{axis}[
                        xtick={0,5,10,15},
                        yticklabels={,,},
                        scaled y ticks={base 10:0},
                        legend pos=north east,
                        ymajorgrids=true,
                        grid style=dashed,
                        legend style={font=\tiny},
                    ]
                    \addplot[
                        color=magenta,
                        mark=*,
                    ]
                    coordinates {
(0,18739) (1,8697) (2,1390) (3,997) (4,109) (5,96) (6,20) (7,21) (8,4) (9,1) (10,4) (11,0) (12,0) (13,3) (14,0) (15,0) (16,0) (17,0) (18,0) (19,0)                  };
                    \legend{folk}

                    \end{axis}
                \end{tikzpicture}
                \label{}
            }\hspace{0pt}%
            \subfigure{
                \begin{tikzpicture}
                    \begin{axis}[
                        xtick={0,5,10,15},
                        yticklabels={,,},
                        legend pos=north east,
                        ymajorgrids=true,
                        grid style=dashed,
                        legend style={font=\tiny},
                    ]
                    \addplot[
                        color=cyan,
                        mark=triangle*,
                    ]
                    coordinates {
(0,8047) (1,594) (2,248) (3,157) (4,99) (5,61) (6,50) (7,51) (8,35) (9,21) (10,19) (11,16) (12,11) (13,10) (14,14) (15,4) (16,0) (17,0) (18,1) (19,0)                     };
                    \legend{GAN}

                    \end{axis}
                \end{tikzpicture}
            }\hspace{0pt}%
            \subfigure{
                \begin{tikzpicture}
                    \begin{axis}[
                        xtick={0,5,10,15},
                        yticklabels={,,},
                        legend pos=north east,
                        ymajorgrids=true,
                        grid style=dashed,
                        legend style={font=\tiny},
                    ]
                    \addplot[
                        color=cyan,
                        mark=triangle*,
                    ]
                    coordinates {
(0,7970) (1,1080) (2,389) (3,232) (4,112) (5,73) (6,49) (7,29) (8,27) (9,25) (10,13) (11,10) (12,10) (13,7) (14,3) (15,0) (16,0) (17,0) (18,0) (19,0)                     };
                    \legend{MLE}

                    \end{axis}
                \end{tikzpicture}
                \label{}
            }\hspace{0pt}%
            \subfigure{
                \begin{tikzpicture}
                    \begin{axis}[
                        xtick={0,5,10,15},
                        yticklabels={,,},
                        legend pos=north east,
                        ymajorgrids=true,
                        grid style=dashed,
                        legend style={font=\tiny},
                    ]
                    \addplot[
                        color=cyan,
                        mark=triangle*,
                    ]
                    coordinates {
(0,555) (1,2078) (2,386) (3,291) (4,145) (5,111) (6,74) (7,40) (8,39) (9,17) (10,17) (11,9) (12,8) (13,8) (14,5) (15,3) (16,1) (17,1) (18,1) (19,2)                     };
                    \legend{RL}

                    \end{axis}
                \end{tikzpicture}
                \label{}
            }\hspace{0pt}%
             \subfigure{
                \begin{tikzpicture}
                    \begin{axis}[
                        xtick={0,5,10,15},
                        yticklabels={,,},
                        scaled y ticks={base 10:0},
                        legend pos=north east,
                        ymajorgrids=true,
                        grid style=dashed,
                        legend style={font=\tiny},
                    ]
                    \addplot[
                        color=cyan,
                        mark=triangle*,
                    ]
%
                    coordinates {
(0,13636) (1,2061) (2,596) (3,374) (4,163) (5,106) (6,59) (7,54) (8,47) (9,20) (10,19) (11,14) (12,10) (13,6) (14,7) (15,8) (16,0) (17,4) (18,3) (19,2)                     };
                    \legend{MC}
                    \end{axis}
                \end{tikzpicture}
                \label{}
            }\hspace{0pt}%
            \subfigure{
                \begin{tikzpicture}
                    \begin{axis}[
                        xtick={0,5,10,15},
                        yticklabels={,,},
                        legend pos=north east,
                        ymajorgrids=true,
                        grid style=dashed,
                        legend style={font=\tiny},
                    ]
                    \addplot[
                        color=cyan,
                        mark=triangle*,
                    ]
                    coordinates {
(0,133) (1,119) (2,121) (3,139) (4,101) (5,124) (6,117) (7,121) (8,123) (9,127) (10,130) (11,105) (12,135) (13,125) (14,131) (15,132) (16,136) (17,128) (18,135) (19,123)                     };
                    \legend{RM}
                    \end{axis}
                \end{tikzpicture}
            }\hspace{0pt}%
            \subfigure{
                \begin{tikzpicture}
                    \begin{axis}[
                        xtick={0,5,10,15},
                        yticklabels={,,},
                        legend pos=north east,
                        ymajorgrids=true,
                        scaled y ticks={base 10:0},
                        grid style=dashed,
                        legend style={font=\tiny},
                    ]
                    \addplot[
                        color=violet,
                        mark=*,
                    ]
                    coordinates {
(0,11181) (1,508) (2,48) (3,11) (4,5) (5,3) (6,3) (7,2) (8,0) (9,0) (10,0) (11,0) (12,1) (13,0) (14,0) (15,0) (16,0) (17,0) (18,0) (19,0)                };
                    \legend{Fusion}
                    \end{axis}
                \end{tikzpicture}
                \label{}
            }\hspace{0pt}%
            \subfigure{
                \begin{tikzpicture}[trim axis right]
                    \begin{axis}[
                        scaled y ticks={base 10:0},
                        xtick=data,
                        xticklabels={C,, D,, E,,F\#,,G\#,,A\#,},
                        yticklabels={,,},                        
                        legend pos=south west,
                         ymajorgrids=true,
                        grid style=dashed,
                    ]
                    
                    \addplot[
                        color=orange,
                        mark=*,
                    ]
                    coordinates {
(0,13948) (1,8166) (2,11687) (3,11631) (4,8240) (5,15458) (6,6584) (7,14181) (8,10243) (9,10258) (10,14071) (11,7061)                     };
                    
                    \end{axis}
                \end{tikzpicture}
            }\hspace{0pt}%
            \subfigure{
                \begin{tikzpicture}
                    \begin{axis}[
                        xtick=data,
                        xticklabels={C,, D,, E,,F\#,,G\#,,A\#,},
                        yticklabels={,,}
                        legend pos=south west,
                        ymajorgrids=true,
                        grid style=dashed,
                    ]
                    
                    \addplot[
                        color=magenta,
                        mark=*,
                    ]
                    coordinates {
(0,5169) (1,427) (2,5303) (3,372) (4,4155) (5,2145) (6,1286) (7,5024) (8,492) (9,4677) (10,979) (11,3921)                     };
                    \end{axis}
                \end{tikzpicture}
                \label{}
            }\hspace{0pt}%
            \subfigure{
                \begin{tikzpicture}
                    \begin{axis}[                  
                        xtick=data,
                        xticklabels={C,, D,, E,,F\#,,G\#,,A\#,},
                        yticklabels={,,}
                        legend pos=south west,
                        ymajorgrids=true,
                        grid style=dashed,
                    ]
                    \addplot[
                        color=cyan,
                        mark= triangle*,
                    ]
                    coordinates {
(0,869) (1,629) (2,448) (3,1650) (4,368) (5,863) (6,985) (7,503) (8,1020) (9,369) (10,959) (11,699)                     };
                  
                    \end{axis}
                \end{tikzpicture}
                \label{}
            }\hspace{0pt}%
            \subfigure{
                \begin{tikzpicture}
                    \begin{axis}[
                        xtick=data,
                        xticklabels={C,, D,, E,,F\#,,G\#,,A\#,},
                        yticklabels={,,},
                        legend pos=south west,
                        ymajorgrids=true,
                        grid style=dashed,
                    ]
                     \addplot[
                        color=cyan,
                        mark= triangle*,
                    ]
                    coordinates {
(0,1170) (1,538) (2,814) (3,629) (4,578) (5,1010) (6,481) (7,1099) (8,572) (9,749) (10,868) (11,557)                     };

                    \end{axis}
                \end{tikzpicture}
                \label{}
            }\hspace{0pt}%
            \subfigure{
                \begin{tikzpicture}
                    \begin{axis}[
                        xtick=data,
                        xticklabels={C,, D,, E,,F\#,,G\#,,A\#,},
                        yticklabels={,,},
                        legend pos=south west,
                        ymajorgrids=true,
                        grid style=dashed,
                    ]
                    \addplot[
                        color=cyan,
                        mark= triangle*,
                    ]
                    coordinates {
(0,5270) (1,1) (2,134) (3,92) (4,629) (5,151) (6,501) (7,279) (8,3) (9,117) (10,22) (11,356)                     };
                    
                    \end{axis}
                \end{tikzpicture}
                \label{}
            }\hspace{0pt}%
            \subfigure{
                \begin{tikzpicture}
                    \begin{axis}[
                        xtick=data,
                        xticklabels={C,, D,, E,,F\#,,G\#,,A\#,},
                        yticklabels={,,},                   
                        legend pos=south west,
                        ymajorgrids=true,
                        grid style=dashed,
                    ]

                    \addplot[
                        color=cyan,
                        mark= triangle*,
                    ]
%
%
                    coordinates {
(0,2104) (1,926) (2,1915) (3,1287) (4,1364) (5,1811) (6,888) (7,2099) (8,1162) (9,1608) (10,1640) (11,1185)        };

                    \end{axis}
                \end{tikzpicture}
                \label{}
            }\hspace{0pt}%
            \subfigure{
                \begin{tikzpicture}
                    \begin{axis}[
                        xtick=data,
                        legend pos=south west,
                        xticklabels={C,, D,, E,,F\#,,G\#,,A\#,},
                        yticklabels={,,},
                        ymajorgrids=true,
                        grid style=dashed,
                    ]
                    
                   \addplot[
                        color=cyan,
                        mark= triangle*,
                    ]
                    coordinates {
(0,258) (1,227) (2,198) (3,216) (4,221) (5,176) (6,244) (7,218) (8,169) (9,223) (10,209) (11,205)                    };
                    \end{axis}
                \end{tikzpicture}
                \label{}
            }\hspace{0pt}%
            \subfigure{
                \begin{tikzpicture}
                    \begin{axis}[
                        xtick=data,
                        xticklabels={C,, D,, E,,F\#,,G\#,,A\#,},
                        yticklabels={,,},
                        legend pos=south west,
                        ymajorgrids=true,
                        grid style=dashed,
                    ]
                    \addplot[
                        color=violet,
                        mark=*,
                    ]
                    coordinates {
(0,3849) (1,135) (2,3948) (3,219) (4,1314) (5,1200) (6,172) (7,2941) (8,115) (9,487) (10,147) (11,328)                     };
                    \end{axis}
                \end{tikzpicture}
            }\hspace{0pt}%
        }    
        \caption{upper line: \textbf{DD} (\textbf{x-axis}: duration length;\textbf{y-axis}: percentage); lower line: \textbf{NPD} (\textbf{x-axis}: pitch class; \textbf{y-axis}: percentage) 
        }
        \label{fig:duradis}
\end{figure*}
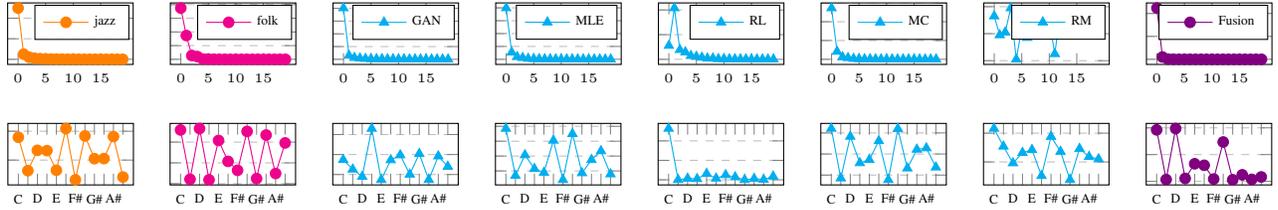

Training data includes two datasets called the Weimar jazz dataset\footnote{http://jazzomat.hfm-weimar.de} and the Essen Associative Code (EsAC) folk dataset\footnote{http://www.esac-data.org}. Both of these datasets contain symbolic data of single-voiced tracks, which support our objective of generating solo melodies. The baseline methods include:

\begin{enumerate}
    \item \textbf{Random Mixing (RM)}: \textbf{RM} randomly selects notes and duration from the pitch and duration distribution on the combination of $\XA$ and $\XB$. 
    \item \textbf{Monte Carlo Sampling (MC)}: After calculating the note pitch and duration distributions on the combination of $\XA$ and $\XB$, \textbf{MC} samples notes and duration subject to these distributions. 
    \item \textbf{Maximum Likelihood Estimation (MLE)}: Employing LSTM, \textbf{MLE} directly learns a pattern from the mixture of $\XA$ and $\XB$ by predicting the next element.
    \item \textbf{GAN}: The model directly learns a pattern from the mixture of $\XA$ and $\XB$ using GAN.
    \item \textbf{Reinforcement Learning (RL)}:  After applying GAN on $\DA-\GA, \DB-\GB$, we exchange the discriminators $\DA$ and $\DB$ which are treated as rewards function to $\GB$ and $\GA$ respectively.
    \end{enumerate}

\subsection{Quantitative Study}
Due to the aesthetic nature of the task, a quantitative evaluation can only give a first impression of the power of the generative system. Here, we investigate the distance between pitch and note duration distributions, knowingly discarding all sequential information of the melodies.

The ideal output should keep equally close to the given domains and minimize this distance. Two distributions of properties are selected for the study.
\begin{enumerate}
    \item \textbf{Duration Distribution (DD)}: Notes are categorized by their durations.
    \item \textbf{Normalized Pitch Distribution (NPD)}: All tokens are categorized into C/C\#/D/D\#/E/F/F\#/G/G\#/A/A\#/B. 
\end{enumerate}

To quantify the similarity between distributions, two symmetric metrics, the Euclidean distance (\textbf{EUD}) and the Wasserstein-1 metric (Earth Mover distance (\textbf{EM})), are employed to calculate the distance between the distributions of generated sequences and those of domains $\XA$ and $\XB$. Intuitively, the distance between $\XA$ and $\XB$, between $\XA, \D_{F}$,  and between $\XB, \D_{F}$ should satisfy the triangle inequality, and the optimal configuration should minimize the inequality, i.e.,
{\footnotesize
\begin{equation*}
\label{sdpfim}
\begin{aligned}
& \underset{}{\text{min}}
& & Diff = \|\XA- \GF(z)\|+\|\XB-\GF(z)\| - \|\XA - \XB\|, \\
& \text{}
&& Ratio =\frac{\|\XA- \GF(z)\| + \|\XB-\GF(z)\|}{\|\XA - \XB\| }.  \\
\end{aligned} 
\end{equation*}
}

\begin{table}[!hpbt]
\centering
\caption{Distance between Distributions (the lower the better)}
\label{tab:distcomp}
\scalebox{1}{
\begin{tabularx}{\linewidth}{c|XX|XX}
\toprule
   & \multicolumn{2}{l|}{\textbf{EUD}} & \multicolumn{2}{l}{\textbf{EM}} \\ \cline{2-5} 
         & Diff          & Ratio           & Diff          & Ratio          \\ \hline
\multicolumn{5}{c}{\textbf{DD}}                            \\ \hline
\textbf{RM} & 39742.2 & 1.375 & 2757.6 & 1.461 \\
\textbf{MLE} & 24546.4 & 1.231 & 2005.2 & 1.335 \\
\textbf{GAN} & 24765.2 & 1.233 & 2064.3 & 1.345 \\
\textbf{RL} & 37971.2 & 1.358 & 2629.0 & 1.439 \\
\textbf{MC} & 13988.7 & 1.132 & 1289.2 & 1.215 \\
\hline
\textbf{Fusion} & 19452.6 & 1.183 & 1831.9 & 1.306 \\
\hline
\multicolumn{5}{c}{\textbf{NPD}}                            \\ \hline
\textbf{RM} &  19586.6 & 1.647 & 5231.0 & 1.643 \\
\textbf{MLE} & 15921.4 & 1.526 & 4147.5 & 1.510 \\
\textbf{GAN} & 16807.1 & 1.555 & 4098.0 & 1.504 \\
\textbf{RL}  & 16927.1 & 1.559 & 4399.2 & 1.541 \\
\textbf{MC} & 11175.0 & 1.369 & 2660.2 & 1.327 \\
\hline
\textbf{Fusion} & 11564.6 & 1.382 & 3182.5 & 1.391 \\
\hline
\bottomrule
\end{tabularx}
}
\end{table}

Figure~\ref{fig:duradis} (top) shows the note duration distributions (DD) and Table \ref{tab:distcomp} (top) gives the corresponding distance measures. The note duration distribution of jazz and folk have shapes that resemble the power law distribution. The shape of \textbf{RM} is very different from jazz or folk. Visually, \textbf{RL}'s duration distribution does not match that of jazz and folk because the portion of notes at the length of 0 and 1 (the peak) were relatively higher than that of jazz or folk. The other baselines (\textbf{GAN}, \textbf{MLE}, \textbf{MC}) appear similar to the given domains because they conform to the power law distributions. \textbf{Fusion} is roughly at the same level of \textbf{MC}. Note that, given our evaluation metrics, \textbf{MC} has the best performance in this experiment since it directly samples notes from the true distribution. This disregard of sequential information, however, results in unnatural rhythmic qualities of the music clips generated by \textbf{MC}. 

Figure~\ref{fig:duradis} (bottom) shows the pitch distributions (NPD) and Table \ref{tab:distcomp} (bottom) gives the corresponding distance measures. 
\textbf{RL}'s pitch distribution is ---due to overfitting--- different from that of jazz and folk, which shows its low capacity in modeling pitch distribution. The single peak at D\# as generated by \textbf{GAN} does not match either the jazz or folk distributions. The distributions of the other baselines are difficult to categorize as a good or bad match. Since \textbf{MC} is directly sampled from real distributions, its configurations should be exactly the same as that of the optimal fusion. \textbf{MC} shows two valleys at C\# and F\#, and two peaks at C and G, which should be expected. \textbf{RM} satisfies one of them (F\# is a valley), while \textbf{RM} meets only one standard (C is the peak). \textbf{GAN} fails in all the characteristic pitch classes, and both \textbf{MLE} and \textbf{Fusion} meet all of them. Quantitatively, the ratio of \textbf{RM} in \textbf{NPD} is the highest which means that \textbf{RM} is the worst in modeling pitch. \textbf{GAN} and \textbf{RL} improve \textbf{RM} by around 10\%. As one of the two distributions that look most similar to the original distributions, \textbf{MLE} outperforms \textbf{RM}, \textbf{RL}, and \textbf{GAN}. Our method \textbf{Fusion} has the second lowest ratio after \textbf{MC}. However, as mentioned above, the \textbf{MC} does not consider the sequential consistence, resulting in bad evaluations in the listening test described in Sect.~\ref{case}.

\subsection{User Study via Listening Test}\label{case}
To evaluate the subjective quality of the generated sequences, a user study was conducted via Amazon Mechanical Turk. The listening test was divided into three steps: 

\begin{enumerate}
    \item \textbf{Qualification test}: \textit{``Listen to the music, and choose one genre below that most matches the music''}. This is to test the evaluators' qualification level. The provided choices are \textit{(A) jazz, (B) folk, (C) neither}.
    \item \textbf{Fusion recognition}: \textit{``Listen to the attached music, and then choose the choice that most matches the music''} Possible answers are \textit{(A) pure jazz, (B) pure folk, (C) mixture of jazz and folk, (D) neither}. Each of the pieces of music generated by our proposed method and the baselines are assigned to one such question.
    \item \textbf{Musicality}: \textit{``Listen to the attached music in 2nd step, who do you think its composer''}, the allowed answers are \textit{(A) expert, (B) newbie, (C) robot}. Similarly, each generated sample is associated with one such question.
\end{enumerate}
We randomly drew from the training data and prepared 800 sets for the first question. In the first step, 66.5\% of candidates chose the correct answer, which is an acceptable rate given the inherent difficulties of genre identification for average listeners. The statistics resulting from the valid answers are shown in Table \ref{amt}. 
Each of the methods generated 500 samples, which were randomly selected for the second and third questions. 

As shown in Table \ref{amt}, the best system based on the percentage of \textit{mixture} is \textbf{MLE}. A closer look at \textbf{MLE}, however, reveals the unbalanced distribution between \textit{jazz} and \textit{folk}, which implies its bias towards the jazz genre. On the other hand, \textbf{RM} is the best system based on the balance between \textit{jazz} and \textit{folk}, but the high percentage of \textit{neither} suggests the confusion of the listeners. It is clear that a single criterion is insufficient for determining the best system.
Therefore, we design a metric to summarize the results and represent the fusion level ($\mathbf{FL}$) of the evaluated systems with the following equation:
{\small
\begin{equation*}
\label{fusionlvl}
\begin{aligned}
& \underset{}{\text{maximize}}
& & 
\mathbf{FL}=1-\frac{\|C_{jazz}-C_{folk}\|+C_{neither}}{C_{jazz}+C_{folk}+C_{mixture}+C_{neither}},\\
\end{aligned} 
\end{equation*}
}where $C_{i}$ indicates the count of $i$. $\|C_{jazz}-C_{folk}\|$ means the unbalanced error, and $C_{neither}$ can be treated as another type of error. A higher $\mathbf{FL}$ value implies better fusion.\\
In the $Musicality$ test, \textbf{MLE} (45.5\%), \textbf{GAN} (42.0\%), and \textbf{Fusion} (43.8\%) were voted by the majority as \textit{expert}. Overall, \textbf{Fusion} shows a balanced performance in both $Fusion recognition$ and $Musicality$ tests, which demonstrates the effectiveness of our proposed method.

\begin{table}
\caption{Listening test results 
}
\label{amt}
\begin{center}
\vspace{-0.5em}
 \begin{tabularx}
 {\linewidth}{c|>{\hsize=.7\hsize}X >{\hsize=.7\hsize}X >{\hsize=.7\hsize}XX|X}
 \toprule
 
\multicolumn{6}{c}{\textbf{Fusion recognition}}                            \\ \hline
    & \textit{jazz} & \textit{folk} & \textit{mixture} & \textit{neither} &\textbf{FL}\\ \hline
\textbf{RM}   & 25.0\% & 22.5\% & 12.5\%  & 40\%& 57.5\%\\ \hline
\textbf{MLE}   & 43.6\% & 9.1\% & 30.9\%  & 16.4\% & 49.1\%\\ \hline
\textbf{GAN}   & 34.0\% & 17.0\% & 26.0\% &14\%& 69\%\\ \hline
\textbf{RL}   & 20.1\% & 28.3\% & 20.8\%  & 30.8\%& 61\% \\ \hline
\textbf{MC}   & 32.0\% & 2.0\% & 14.0\% & 52\%& 16\% \\ \hline
\textbf{Fusion}   & 35.9\% & 25.0\% & 20.0\% & 19.1\% & 70\% \\ \hline
\end{tabularx}
\smallskip

\begin{tabularx}{\linewidth}{c|X X X}
\multicolumn{4}{c}{\textbf{Musicality}}                            \\ \hline
     & \textit{expert} & \textit{newbie} & \textit{robot} \\ \hline

\textbf{RM}   & 22.5\% & 50.0\% & 27.5\%  \\ \hline
\textbf{MLE}   & 45.5\% & 21.8\% & 32.7\%  \\ \hline
\textbf{GAN}   & 42.0\% & 32.0\% & 26.0\%  \\ \hline
\textbf{RL}   & 32.1\% & 37.7\% & 30.2\%   \\ \hline
\textbf{MC}   & 30.0\% & 36.0\% & 34.0\% \\ \hline
\textbf{Fusion}   & 43.8\% & 28.1\% & 28.1\% \\
\bottomrule
\end{tabularx}
\vspace{-0.5em}
\end{center}
\end{table}
\vspace{-0.5em}

\section{Conclusion}\label{conclusion}
In this paper, we proposed a three-way GAN-based learning framework to integrate multiple domains. A Wasserstein distance based metric is introduced to indicate the blending progress. The evaluation investigated objective metrics looking at the pitch and note duration distributions of the generated data. A user study explicitly illustrates the validity of the proposed method as compared to the baselines, as the melodies generated by our model were preferred by the majority of users. While the listening test results are encouraging, future work will benefit from removing key dependence of the training data, a careful look into perceptually meaningful evaluation metrics that take into account the sequential nature of music, and a careful listening test design, including human-composed melodies for comparison.

\bibliographystyle{IEEEtranS}
\bibliography{IEEEexample}

\end{document}